\documentclass{article}
\usepackage[utf8]{inputenc}
\usepackage{amsfonts}
\usepackage{amssymb}
\usepackage{xcolor}
\usepackage{subcaption}
\usepackage{algorithm}
\usepackage{array,multirow,graphicx}
\usepackage{multirow}
\usepackage{graphicx}
\usepackage{lscape}
\usepackage{amsthm}
\usepackage{amsmath}
\DeclareMathOperator*{\argmax}{argmax}
\DeclareMathOperator*{\argmin}{argmin}
\usepackage{authblk}

\theoremstyle{plain}
\newtheorem{theorem}{Theorem}[section]

\theoremstyle{definition}
\newtheorem{definition}[theorem]{Definition}

\theoremstyle{remark}

\author{
  Alkhouri, Ismail\\
  \small{Department of Electrical and Computer Engineering,\\
  University of Central Florida}
  \and
  Atia, George\\
  \small{Department of Electrical and Computer Engineering, \\
  University of Central Florida}
  \and
    Velasquez, Alvaro\\
  \small{{Information Directorate, Air Force Research Laboratory}}
}

\title{A Differentiable Approach to Combinatorial Optimization using Dataless Neural Networks}

\begin{document}

\maketitle




\begin{abstract}
    The success of machine learning solutions for reasoning about discrete structures has brought attention to its adoption within combinatorial optimization algorithms. Such approaches generally rely on supervised learning by leveraging datasets of the combinatorial structures of interest drawn from some distribution of problem instances. Reinforcement learning has also been employed to find such structures. In this paper, we propose a radically different approach in that no data is required for training the neural networks that produce the solution. In particular, we reduce the combinatorial optimization problem to a neural network and employ a dataless training scheme to refine the parameters of the network such that those parameters yield the structure of interest. We consider the combinatorial optimization problems of finding maximum independent sets and maximum cliques in a graph. In principle, since these problems belong to the \textbf{NP-hard} complexity class, our proposed approach can be used to solve any other \textbf{NP-hard} problem. \textcolor{black}{Additionally, we propose a universal graph reduction procedure to handle large scale graphs. The reduction exploits community detection for graph partitioning and is applicable to any graph type and/or density.} Experimental evaluation on both synthetic graphs and real-world benchmarks demonstrates that our method performs on par with or outperforms state-of-the-art heuristic, reinforcement learning, and machine learning based methods without requiring any data.
\end{abstract}

\section{Introduction}

In his seminal work \cite{karp1972reducibility}, Richard Karp demonstrated the reducibility among combinatorial problems that are complete for the complexity class \textbf{NP}. Combinatorial optimization problems have since been frequently associated with the \textbf{NP-hard} complexity class for which no efficient solutions are likely to exist. Despite their apparent intractability, \textbf{NP-hard} problems have found ubiquitious application across many sectors \cite{bengio2021machine}. 


A characteristic feature of this class of optimization problems is that an instance of one problem can be reduced in polynomial time to another \textbf{NP-hard} problem. Consequently, if a solution is found for one, it can be obtained for the other \textcolor{black}{\cite{li2018combinatorial}}. While there is not a known polynomial time solver with respect to (w.r.t.) the size of the input for any of these problems, there are many approximate and efficient solvers \cite{lamm2016finding}. In general, these solvers are broadly categorized into heuristic algorithms \cite{akiba2016branch}, conventional branch-and-bound methods \cite{san2011exact}, and approximation algorithms \cite{hochba1997approximation}.\par 


Recently, learning-based heuristics have emerged as effective means to solve combinatorial optimization problems \cite{he2014learning,li2018combinatorial}. However, these methods require extensive training of Neural Networks (NNs) using large graph datasets with known solutions. In contrast, we introduce dataless NNs (dNNs) and dataless training as a novel paradigm for solving \textbf{NP-hard} problems. Given a graph $G$, 
the key idea underlying our approach is reducing the combinatorial optimization problem to a NN with a connectivity structure derived from $G$, and whose input is data independent. The output of the NN is minimized upon finding a desired structure (e.g., a maximum independent set) and this structure can be constructed from the learned parameters of the NN.\par



The first contribution of this paper is the introduction of dNNs for which no data is required during training. The second contribution is the representation of \textbf{NP-hard} problems as a single differentiable function, thereby enabling the adoption of differentiable solutions to classic discrete optimization problems. \textcolor{black}{Third}, we construct dNN architectures for solving the maximum independent sets (MIS), maximum cliques (MC), and minimum vertex covers (MVC) problems. These problems can directly model many problems of interest. \textcolor{black}{Fourth, we develop a community detection based graph reduction procedure for large scale graphs. Unlike most common reductions rules whose applicability is limited to sparse graphs, this procedure is universal in that it is applicable to any graph type.} \textcolor{black}{Our fifth contribution is the introduction of an iterative solution improvement procedure based on simulated annealing and dNNs.} 
We evaluate our proposed approach using SNAP, citation networks benchmarks, and \textcolor{black}{synthetic graphs}, where we demonstrate that it performs on par with the state-of-the-art heuristic, reinforcement learning (RL), and machine learning (ML) based baselines with the added benefit that our solution does not require data.

\section{Related Work}

Exact algorithms for \textbf{NP-hard} problems are typically based on enumeration or branch-and-bound techniques. 
However, these techniques are not applicable to large problem spaces \cite{dai2016discriminative}. This motivated the development of efficient approximation algorithms and heuristics, such as the procedure implemented in the NetworkX library for solving MIS \cite{boppana1992approximating}. These polynomial time algorithms and heuristics typically utilize a combination of various sub-procedures, 
including greedy algorithms, local search sub-routines, and genetic procedures \cite{williamson2011design}. 
An algorithm that provably guarantees an approximate solution to MC within a factor of $n^{1 - \epsilon}$, where $n$ is the number of nodes in the underlying graph, for any $\epsilon > 0$ is not possible unless \textbf{P = NP} \cite{hastad1996clique}. Similar inapproximability results have been established for the MIS problem \cite{berman1992complexity}. As such, heuristics without approximation guarantees have been adopted for practical purposes for these problems.

The ReduMIS method \cite{lamm2016finding} is the state-of-the-art solver for the MIS problem. It consists of two main components: an iterative implementation of a series of graph reduction techniques, 
followed by the use of an evolutionary algorithm. The latter starts with a pool of independent sets, then evolves the pool 
over several rounds. In each round, the algorithm uses a selection procedure to select favorable nodes. This is achieved by executing graph partitioning that clusters the graph nodes into disjoint clusters and separators for solution improvement. Our method, however, does not include this expensive solution combination operation. \textcolor{black}{In contrast, our use of community detection is altogether different as it takes place prior to obtaining the initial solution(s) for the purpose of scaling up to large graphs and not for solution improvement. Moreover, we do not enforce the partitions and separators to be disjoint (i.e., sharing no edges)} (See Section \ref{sec:CD}).

Learning-based approaches which make use of RL algorithms and ML architectures have been recently introduced to solve \textbf{NP-hard} problems. RL-based methods train a deep Q-Network (DQN) such that the obtained policy operates as a meta-algorithm that incrementally yields a solution \cite{bengio2021machine}. 
The recent state-of-the-art work of \cite{dai2017learning} combines a DQN with graph embeddings, 
allowing discrimination between vertices based on their impact on the solution, and enabling scalability to larger problem instances.
By contrast, our proposed method does not require training of a DQN. The supervised learning method in \cite{li2018combinatorial} achieves state-of-the-art performance for the MIS problem. It integrates several graph reductions \cite{lamm2016finding}, Graph Convolutional Networks (GCN) \cite{defferrard2016convolutional}, guided tree search, and a solution improvement local search algorithm \cite{andrade2012fast}. 
The GCN is trained using benchmark graphs 
and their solutions as the true labels 
to learn probability maps 
for each vertex being in the optimal solution. The point of resemblance to our approach is the use of 
an NN 
to derive solutions to combinatorial optimization problems. However, a major difference is that our approach does not rely on supervised learning; it uses an entirely different dNN and obtains a solution via dataless training. More specifically, the means by which we optimize the dNN consists of applying backpropagation \cite{riedmiller1993direct} to a loss function defined entirely in terms of the given graph and the structure of the dNN without the need for a dataset as is standard in training deep learning models.



\section{Preliminaries}

An undirected graph is denoted by $G=(V,E)$, where $V$ is its vertex set and $E\subseteq V \times V$ is its edge set. The number of nodes is $|V| = n$ and the number of edges is $|E| = m$. 
We also use the notation $V(H)$ and $E(H)$ to refer to the vertex and edge sets of some graph $H$, respectively. 
The degree of a node $v\in V$ is denoted by $\textrm{d}(v)$, and the maximum degree of the graph by $\Delta(G)$. The neighborhood of node $v\in V$ is $N(v) = \{u\in V \mid (u,v)\in E\}$. 
For a subset of nodes $U\subseteq V$, $G[U] = (U, E[U])$ is used to represent the subgraph induced by $U$, i.e., the graph on $U$ whose edge set $E[U] = \{(u, v) \in E \mid u, v \in U\}$ consists of all edges of $G$ with both ends in $U$. The complement of graph $G$ is the graph $G'=(V,E')$ on $V$, where $E' = V \times V \setminus E$, 
i.e., $E'$ consists of all the edges between nodes that are not adjacent in $G$, with $|E'| = m'$. Hence, $m+m' = n(n-1)/2$ is the number of edges in the complete graph on $V$. For any positive integer $n$, $[n]:=\{1, \ldots, n\}$. We use $|\cdot|$ to denote the cardinality of a set, unless stated otherwise. 





\textcolor{black}{We consider the \textbf{NP-hard} problem of finding maximum independent sets (MIS). We define the MIS problem and the related maximum clique (MC) and minimum vertex cover (MVC) problems, then briefly describe how MC and MVC can be represented as instances of MIS.}


\begin{definition}[MIS Problem]
Given an undirected graph $G = (V, E)$, MIS is the problem of finding a subset of vertices $\mathcal{I} \subseteq V$ such that $E(G[\mathcal{I}]) = \emptyset$, and $|\mathcal{I}|$ is maximized.
\end{definition}


\begin{definition}[MC Problem]
Given an undirected graph $G = (V, E)$, MC is the problem of finding a subset of vertices $C \subseteq V$ such that $G[C]$ is a complete graph, and $|C|$ is maximized.
\end{definition}

\begin{definition}[MVC Problem]
Given an undirected graph $G = (V, E)$, MVC is the problem of finding a subset of vertices $R \subseteq V$ such that, for every $(u,v)\in E$, either $u\in R$ or $v\in R$, and $|R|$ is minimized.
\end{definition}

For the MC problem, the MIS of a graph is an MC of the complement graph \cite{karp1972reducibility}. MVC and MIS are complementary, i.e, a vertex set is independent if and only if its complement is a vertex cover \cite{cook1995combinatorial}. We exploit these properties in the development of our dNNs.


\section{Methodology}
\begin{figure}[t]
\centering
\includegraphics[width=9cm]{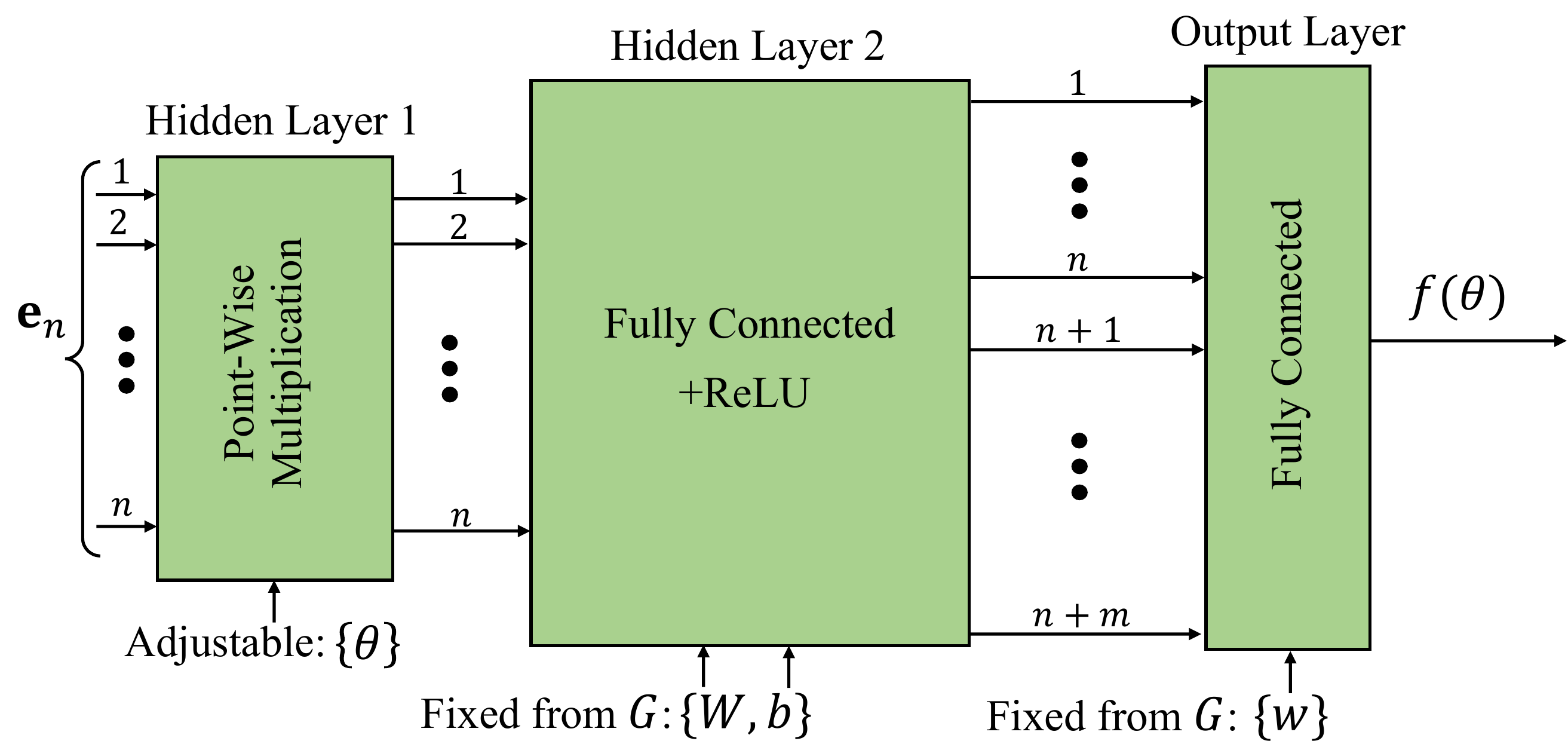}
\caption{{Block diagram of the proposed dNN, $f(\theta)$.}}
\label{fig: BD}
\end{figure}
\begin{figure*}[t]
\centering
\includegraphics[width=12cm]{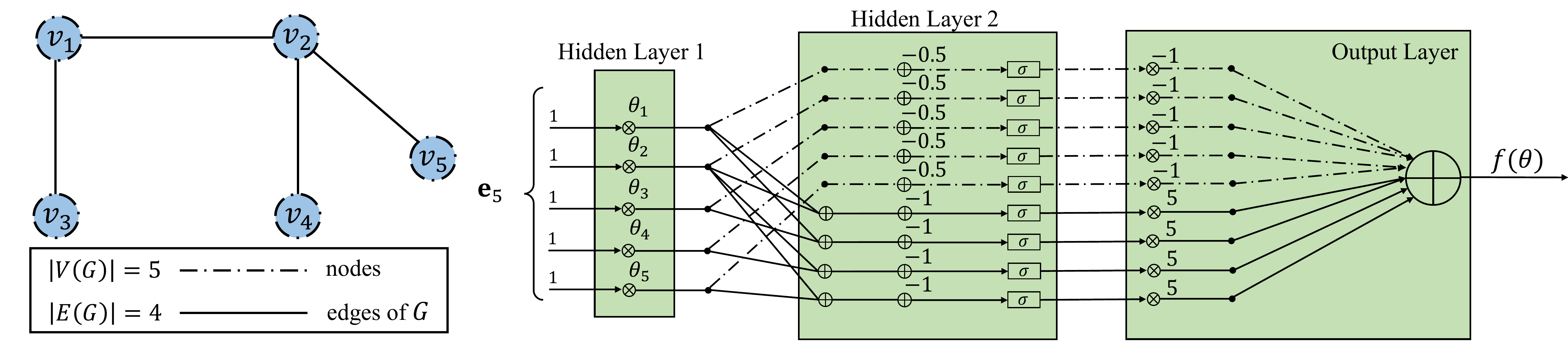}
\caption{{An example of graph $G=(V=\{v_1,v_2,v_3,v_4,v_5\},E=\{(v_1,v_2),(v_1,v_3),(v_2,v_4),(v_2,v_5)\})$ and its dNN construction $f$ for the MIS problem.}}
\label{fig: Graph to NN construction no Clique.}
\end{figure*}
In this section, we describe the different components of our proposed approach. 
To preface the discussion and distinguish our dataless solution from learning-based methods, consider first the conventional supervised learning setting. In this setting, there is generally some set of data $D = \{(x_i, y_i)\}_i$ consisting of input tensors $x_i \in \mathbb{R}^n$ and their associated value, or label, $y_i \in Y$. The goal of learning is then to train a learning model $f : \mathbb{R}^n \rightarrow Y$ parameterized by $\theta$ so that $f$ learns to predict the input-output relationship of the underlying data distribution $D'$. This is given by the objective below, in which $\mathcal{L}$ denotes the loss function. 
\begin{equation}
\label{eqn: 1}
\min_\theta \mathbb{E}_{(x, y) \sim D'} [\mathcal{L}(f(x; \theta), y)]\:
\end{equation}
Since a dataset $D$ is used in lieu of the true underlying data distribution $D'$, the objective function becomes
\begin{equation}
\label{eqn: 2}
\min_\theta \frac{1}{|D|} \sum_{(x, y) \in D} \mathcal{L}(f(x; \theta), y)\:
\end{equation}
The loss function is chosen to be a differentiable function, such as the minimum square error, in order to leverage optimization using backpropagation. In our approach, we leverage dNNs, which we define as neural networks whose loss function $\mathcal{L}$ does not depend on data. In this sense, what we present is a learning technique different from supervised, unsupervised, and reinforcement learning.



\subsection{Dataless Neural Network Construction}

Given a graph $G = (V, E)$, we construct a dNN $f$ with trainable parameters $\theta\in [0,1]^n$ whose single output is $f(\textbf{e}_{n}; \theta) = f(\theta) \in \mathbb{R}$. Note that the input to $f$ is the all-ones vector $\textbf{e}_{n}$ and thus does not depend on any data. The network consists of an input layer, two hidden layers, and an output layer. The trainable parameters $\theta \in [0, 1]^n$ connect the input layer $\textbf{e}_{n}$ to the first hidden layer through an elementwise product. All other parameters are fixed during training and are presented next. The connectivity structure from the first hidden layer to the second is given by the binary matrix $W\in \{0,1\}^{n\times(n+m)}$ and will depend on $G$, the bias vector at the second hidden layer is given by $b\in \{-1,-1/2\}^{n+m}$, and the fully-connected weight matrix from the second hidden layer to the output layer is given by $w \in \{-1,n\}^{n+m}$. These parameters are defined as a function of $G$. The output of $f$ is given by (\ref{equation:networkOutput}), where $\odot$ is the element-wise Hadamard product that represents the operation of the first hidden layer of the constructed network. The second hidden layer is a fully-connected layer with fixed matrix $W$ and bias vector $b$ with a ReLU activation function $\sigma(x) = \max (0, x)$, while the last layer is another fully-connected layer expressed in vector $w$. 
See Figure \ref{fig: BD} for a block diagram of the generalized proposed network. 
\begin{equation} 
f(\textbf{e}_n\:; \theta) = f(\theta) = w^T \sigma(W^T(\textbf{e}_{n} \odot \theta)+b)\:
\label{equation:networkOutput}
\end{equation}
In the sequel, we prove that $f(\theta)$ is an equivalent differentiable representation of the MIS problem $G = (V, E)$ in that $f$ achieves its minimum value when an MIS $U \subseteq V$ is found in $G$. Furthermore, this is a constructive representation since $U$ can be obtained from $\theta$ as follows. Let $\theta^* = \argmin_{\theta\in [0,1]^n} f(\theta)$ denote an optimal solution to $f$ and let $\mathcal{I} : [0, 1]^n \rightarrow 2^V$ denote the corresponding independent set found by $\theta$ such that
\begin{equation} \label{eqn: MIS w.r.t theta}
\mathcal{I} (\theta) = \{v \in V \mid \theta^*_v \geq \alpha \}\:,
\end{equation}
for any $\alpha > 0$. We show that $|\mathcal{I}(\theta^*)| = |U|$. Intuitively, this is tantamount to choosing the trained parameter indices in $\theta$ whose values exceed some threshold and choosing the vertices in $V$ corresponding to those indices to be the MIS.
The fixed parameters of $f$ are constructed from the given graph $G = (V, E)$ as follows. The first $n\times n$ submatrix of $W$ represents the nodes $V$ in the graph and its weights are set equal to the identity matrix $I_n$. The following $m$ columns of $W$ correspond to the edges $E$ in the graph. In particular, for the column associated with a given edge, a value of $1$ is assigned to the entries corresponding to both ends of that edge and $0$ otherwise. 
For the bias vector $b$, we assign a value of $-1/2$ to the entries corresponding to the first $n$ nodes, and a value of $-1$ to the $m$ entries corresponding to the edges. Finally, these nodes are input to their corresponding ReLU activation functions. For the vector $w$ connecting the second hidden layer to the output layer, values of $-1$ and $n$ are assigned to entries corresponding to nodes and edges in the second hidden layer, respectively. Hence, the parameters $W$, $b$, and $w$ are defined as
\begin{equation}
\begin{gathered} \label{eqn: W nodes and edges}
W(i,i) = 1, ~v_i\in V,~ i\in [n]\:, \\
W(i,n+l) = W(j,n+l) = 1\:,
\forall e_l = (v_i,v_j) \in E, l\in [m]\:,
\end{gathered} 
\end{equation}
\begin{equation}
\begin{gathered} \label{eqn: b and nodes and edges}
b(i) = -1/2,~w(i) = -1, ~v_i\in V,~ i\in [n]\:, \\
b(n+l) = -1,~w(n+l) = n,~l\in [m]\:.
\end{gathered} 
\end{equation}
Therefore, we can rewrite (\ref{equation:networkOutput}) as follows 
\begin{equation} \label{eqn: NN h w.r.t. only w 1}
\begin{gathered}
f(\theta) = -\sum_{v\in V} \sigma(\theta_v-1/2) + 
n\sum_{(u,v)\in E}\sigma(\theta_{u}+\theta_{v}-1)\:.
\end{gathered}
\end{equation}
Figure \ref{fig: Graph to NN construction no Clique.} presents an example. The following theorem establishes a relation between the minimum value of \eqref{eqn: NN h w.r.t. only w 1} and the size of the MIS.
\begin{theorem} \label{th: min h}
Given a graph $G = (V, E)$ and its corresponding dNN $f$, an MIS $U \subseteq V$ of $G$ is of size $|U| = k$ if and only if the minimum value of $f$ is $-k/2$.
\end{theorem}

\begin{proof}
($\implies$) Assume that $|U| = k$ and let $\theta_{v_i} = 1$ for each $v_i\in U$ and $\theta_{v_i} = 0$ otherwise. For an arbitrary pair of nodes $v_i, v_j \in U$, consider the output of $f$ as visualized by Figure~\ref{fig:proof figure}, where edge values denote the outputs of the preceding nodes in the network and nodes $\eta^1_i, \eta^2_i$ denote the $i^\text{th}$ neurons in the first and second hidden layers, respectively. We abuse notation to refer to both the output neuron and the output value as $f(\theta)$.

\begin{figure}
    \centering
    \includegraphics[width=10cm]{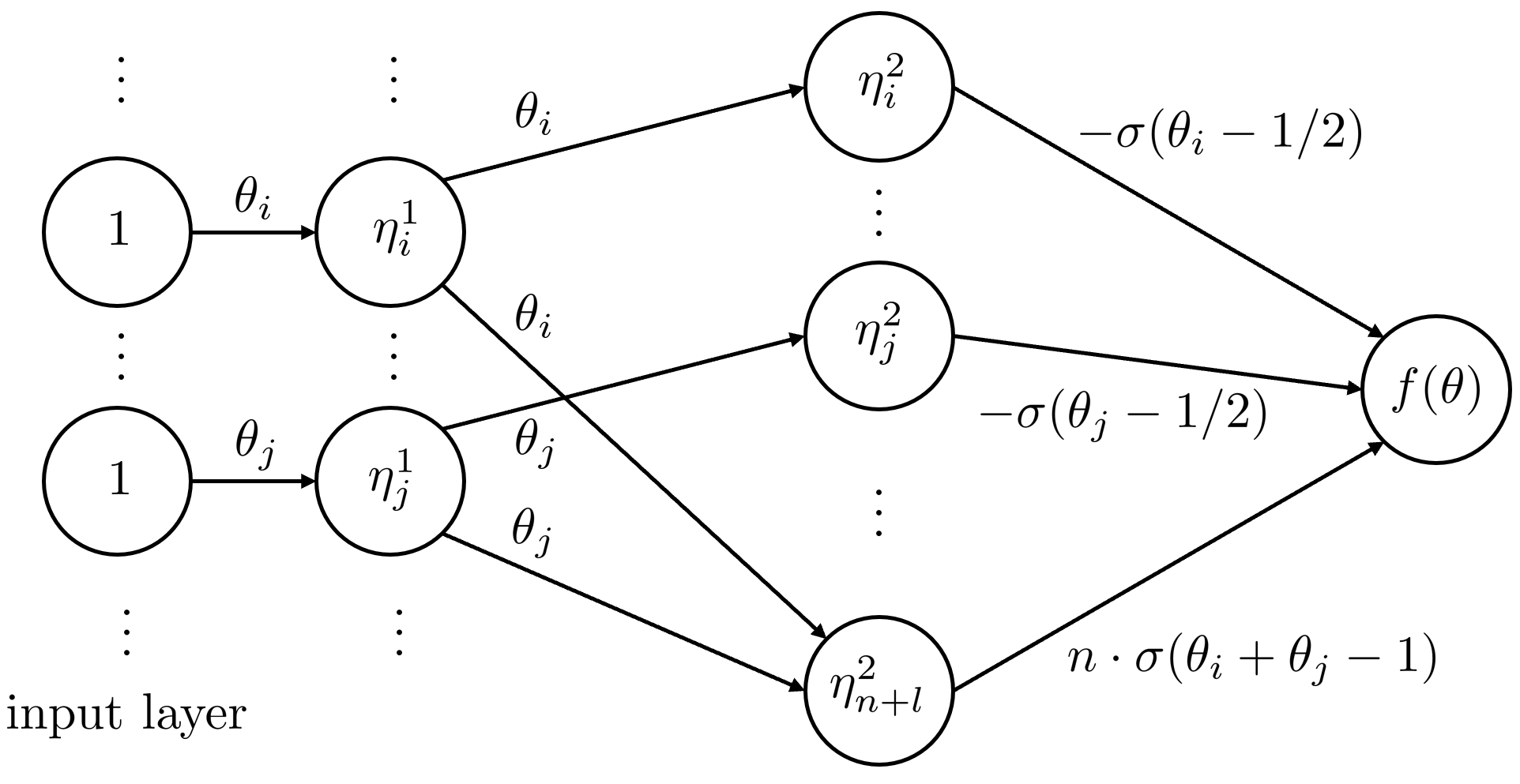}
    \caption{Network $f$.}
    \label{fig:proof figure}
\end{figure}

Note that $v_i$ and $v_j$ each contribute an output of $-1/2$ to $f(\theta)$. This follows from the fact that, by definition of MIS, these vertices do not share an edge and so the output of $\eta^2_{n + l}$ is $0$. Thus, for an MIS of size $|U| = k$, we have $f(\theta) = -k/2$. This is the minimum value attainable by $f$. Indeed, consider, for the sake of contradiction, that there exists $\theta'$ such that $f(\theta') < f(\theta)$. As with $\theta$, this $\theta'$ must be defined such that $\theta_{v_i} = 1$ for each $v_i\in U$. Consider the addition of some other $v_{k'} \notin U$. Then $\eta^2_{k'}$ will contribute $- \sigma(\theta_{k'} - 1/2)$ to $f(\theta')$ and $\eta^2_{n + l}$ will contribute at least $n \theta_{k'}$ for every edge $e_k = (v_{k'}, v) \in E$, where $v \in U$. By definition of MIS, some such edge must exist. Therefore, we would have $f(\theta') > f(\theta)$, yielding a contradiction.



($\impliedby$) Assume that the minimum value of $f$ is $f(\theta) = -k/2$. Consider an arbitrary edge $e_l = (v_i,v_j)\in E$. It follows from the construction of $f$ that $\theta_{v_i}+\theta_{v_j}\leq 1$ must hold for $f$ to achieve its minimum value. For the sake of contradiction, consider that $\theta_i + \theta_j > 1$. In this case, neuron $\eta^2_{n + l}$ contributes $n (\theta_i + \theta_j - 1) > 0$ to the output $f(\theta)$. 
This yields a contradiction since we can simply choose $\theta_i$ and $\theta_j$ to be $0$, thereby contributing a value of $0$ to $f(\theta)$. Given an arbitrary vertex $v$ and its neighbors $N(v)$, it must be the case that $\theta_i = 1$ for some $v_i \in N(v)$ as this would contribute a value of $-1/2$ to $f(\theta)$ through node $\eta^2_i$ and a value of $0$ through nodes $\eta^2_{n + l}$ for all vertices $u \in N(v)$ connected to $v$ through edge $e_l$. It follows that there must be $k$ entries in $\theta$ with value $1$, each contributing a value of $-1/2$ to the output $f(\theta)$ such that none of them share a neuron in the second hidden layer. These correspond to $k$ vertices in $V$ that do not share edges. Therefore, $|U| = k$.
\end{proof}

From Theorem~1, it follows that the minimum value of $f$ is achieved when the maximum number of entries in $\theta$ have value 1 such that their corresponding nodes in $G$ share no edges. This yields an independent set $\mathcal{I}(\theta) = \{v_i \in V \mid \theta_{v_i} = 1\}$ of maximum cardinality.\par

Due to the non-linearity introduced by the ReLUs in the dNN $f$, we obtain a minimizer for $f$ by leveraging the backpropagation algorithm along with the well-known ADAM optimizer \cite{kingma2014adam}. 
Hence, we iteratively minimize the loss function $\mathcal{L}(f(\theta),f_d) = |f(\theta)-f_d|^2 \:,$
where $|\cdot|$ denotes the absolute value and $f_d$ is the minimum desired value used for parameter tuning. \textcolor{black}{Per Theorem~\ref{th: min h}, the minimum achievable value of $f(\theta)$ is a function of the size of the MIS. Therefore, during training we select $f_d = -n/2$, a value that is only attained by $f(\theta)$ if $G$ is a null graph}. 






\subsection{On the Duality of MIS and MC}

Since the graph induced by the MIS is a null graph on $G$ and fully-connected on its complement $G'$, we propose to include the edges of $G'$ in the construction of $f$ to enhance the tuning of the parameters $\theta$. We term the resulting enhanced dataless neural network as $h$ with output value $h(\theta)$. 
In this case, we extend the definition of the fixed parameters 
$W\in \{0,1\}^{n\times(n+m+m')}$, $b\in \{-1,-1/2\}^{n+m+m'}$, and $w \in \{-1,n\}^{n+m+m'}$, 
by defining the mapping for the augmented portion of these parameters representing the $m'$ edges of $G'$ as

%
\begin{equation}
\begin{gathered} \label{eqn: W edge of G hat}
W(i,n+m+l) = W(j,n+m+l) = 1,\\
\forall e_l = (v_i,v_j) \in E(G'), l\in [m']\:,
\end{gathered} 
\end{equation}
\begin{equation}
\begin{gathered} \label{eqn: b and w edges of G hat}
b(n+m+l) = w(n+m+l) =  -1,~l\in [m']\:.
\end{gathered} 
\end{equation}
%
Given \eqref{eqn: NN h w.r.t. only w 1}, \eqref{eqn: W edge of G hat}, and \eqref{eqn: b and w edges of G hat}, the output of $h$ is
\begin{equation} \label{eqn: NN h w.r.t. only w}
\begin{gathered}
h(\theta) = f(\theta)-\sum_{(u,v)\in E(G')}\sigma(\theta_{u}+\theta_{v}-1)\:.
\end{gathered}
\end{equation}
Figure~\ref{fig: Graph to NN construction.} presents an example of the proposed construction from a simple 5-node graph $G$ to its corresponding dNN $h$. Our next result is analogous to Theorem \ref{th: min h} and establishes a relation between the minimum value of \eqref{eqn: NN h w.r.t. only w} and the size of an MIS in $G$. 

\begin{theorem} \label{th: min h hat}
Given a graph $G = (V, E)$ and its corresponding enhanced dNN $h$, an MIS $U \subseteq V$ of $G$ is of size $|U| = k$ if and only if the minimum value of $h$ is $-k^2/2$.
\end{theorem}

\begin{proof}
Per Theorem~\ref{th: min h}, the minimum value of the first term of \eqref{eqn: NN h w.r.t. only w} is $-k/2$. Therefore, we consider the minimum value of the remaining (second) term, which corresponds to the edges of $G'$. 
Similar to Theorem~\ref{th: min h}, assume that $\theta_v = 1$ for each $v \in U$ and $\theta_v = 0$ otherwise. The graph induced by the MIS w.r.t. $G'$ is a fully-connected graph, i.e., $|E(G'[U])| = k(k-1)/2$. Given the $-1$ bias, the outputs associated with the edges of $G'$ will be $1$. Since the subgraph induced on $G'$ is complete, we get $-k(k-1)/2$ for the second term. The combined output is thus $-(k/2)-(k(k-1)/2) = -k^2/2$, which concludes the proof. 
\end{proof}
%
Since the minimum value of $h$ is $-k^{2}/2$, we use $h_d=-n^{2}/2$ for training the dNN by minimizing the loss 
\begin{equation}
\begin{gathered} \label{eqn: loss h}
\mathcal{L}(h(\theta),h_d) = |h(\theta)-h_d|^2 \:.
\end{gathered} 
\end{equation}
\begin{figure*}[t]
\centering
\includegraphics[width=12cm]{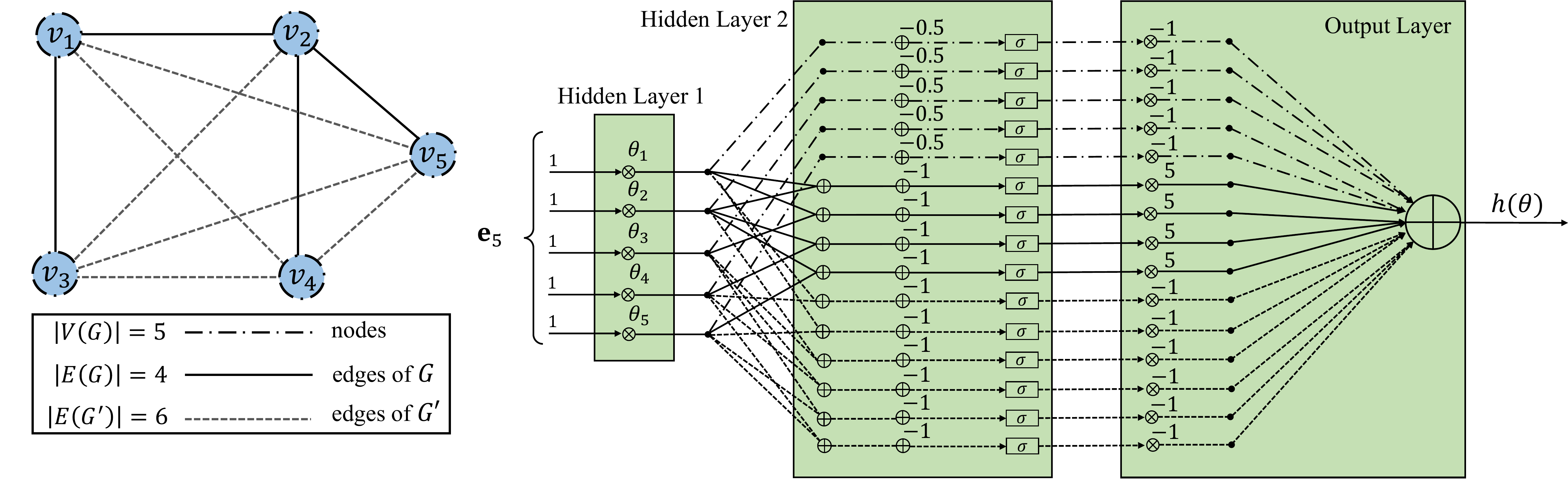}
\caption{{An example of graph $G=(V=\{v_1,v_2,v_3,v_4,v_5\},E=\{(v_1,v_2),(v_1,v_3),(v_2,v_4),(v_2,v_5)\})$ and its dNN construction $h$ for the MIS problem by leveraging the duality between MIS and MC.}}
\label{fig: Graph to NN construction.}
\end{figure*}
In general, the vertices with high degrees are less likely to be part of an MIS than vertices with low degrees. Therefore, to speed up the training of the parameters of $h$, we initialize every element of $\theta$ with a probability that is decreasing in the node degree as 
\begin{equation} \label{eqn: theta initial}
\theta_v = 1-\frac{\textrm{d}(v)}{\Delta(G)} + s\:,\:\:
\theta \leftarrow \frac{\theta}{\max_{v\in V}\theta_v}\:,
\end{equation}
where we add a small value $s$ drawn from a uniform distribution over small positive bounds ($<<0.1$) as part of the ADAM stochastic algorithm to improve performance when optimizing the loss function of $h$ \cite{goodfellow2017deep}. 

\begin{algorithm}[t]
\small
\caption{Finding MIS with dNN}
\textbf{Function}: $\mathcal{I} = \texttt{dNN}(G,\alpha)$ \\
\vspace{1mm}
\small{ 1:}  \textbf{construct} $h$ from $G$ using \eqref{eqn: W nodes and edges}, \eqref{eqn: b and nodes and edges}, \eqref{eqn: W edge of G hat}, and \eqref{eqn: b and w edges of G hat} \\
\vspace{1mm}
\small{ 2:}  \textbf{initialize} $\theta$ as in \eqref{eqn: theta initial}, $\mathcal{I} (\theta) = \emptyset$  \\
\vspace{1mm}
\small{ 3:}  \textbf{while} $\exists v\in V \setminus \mathcal{I} (\theta) \textbf{ s.t. } E(G[\mathcal{I} (\theta) \cup \{v\}]) \neq \emptyset$ \\
\vspace{1mm}
\small{ 4:} \hspace{2mm} \textbf{update}  $\theta \leftarrow \argmin_{\theta \in [0,1]^{n}} |h(\theta)-h_d|^2$\\
\vspace{1mm}
\small{ 5:} \hspace{2mm} \textbf{obtain} $\mathcal{I} (\theta) = \{v\in V \mid \theta_v \geq \alpha \}$ \\
\vspace{-3.5mm}
\label{alg: MIS OSS alg}
\end{algorithm}

\subsection{Scaling up: Community Detection Approach}
\label{sec:CD}

To handle large-scale graphs, many techniques have been introduced in the past, including Linear Programming (LP) reduction, removal of pendant vertices, and other heuristics as presented in \cite{akiba2016branch} and adopted in the latest state-of-the-art methods. However, these techniques are only applicable on sparse graphs as pointed out in \cite{lamm2016finding}. 

This motivates our work here on developing a reduction technique that is independent of the graph type and density. To this end, we perform community detection \cite{yang2016comparative} to partition the graph into communities, which are groups of nodes with dense connections internally and sparser connections between groups. Then, using Algorithm~\ref{alg: MIS OSS alg}, we obtain a MIS for the subgraph induced by each community separately. 
Subsequently, a MIS is obtained for the full graph by processing the identified sets. More specifically, let $C_i, i\in [r]$ denote the set of nodes in community $i$, where $r$ is the total number of communities found by a community detection algorithm. The inter-cluster edge set 
\begin{equation} \label{eqn: inter-cluster edges}
R = \{(u,v)\in E   \mid  u \in C_i , v\in C_j, i\neq j \}\:,
\end{equation}
is the set of edges between nodes in different communities. For every $C_i$, we construct a dNN $h_i$ and obtain an MIS $\mathcal{I}_i = \texttt{dNN}(G_i, \alpha)$, where $G_i = G[C_i]$. The union of these sets is the set $B = \bigcup_{i \in [r]} \mathcal{I}_i$. Note that $B$ is generally not an IS w.r.t. graph $G$ since there could exist edges between nodes in the solution sets of two different communities. We call these the forbidden edges and define the set
\begin{equation} \label{eqn: forbidden edges}
F = \{(u,v)\in R \mid u\in \mathcal{I}_i, v\in \mathcal{I}_j, i\neq j\}\:.
\end{equation}
In order to obtain a MIS w.r.t. $G$, we need to handle all nodes with edges in the set $F$. To this end, we develop the following procedure which processes every pair in $F$ until an IS is obtained w.r.t. $G$. First, we select a pair $(u,v)\in F$, then for every node $q\in \{u,v\}$, we check if it can be replaced by a node from its neighborhood. A candidate replacement, $w\in N(q)$, must be 1-tight, that is $|B \cap N(w)| = 1$. 
If no replacements are found for either $u$ or $v$, we remove the node with the larger number of repetitions in $F$. This is repeated until the set $F$ is empty. The entire procedure is presented in Algorithm~\ref{alg: Handlling}. In the case that the resulting set $\mathcal{I}$ is only an IS w.r.t. $G$, we obtain a MIS by executing Algorithm~\ref{alg: MIS OSS alg} on the subgraph induced by nodes that are neither in the solution nor in its neighborhood. 
In particular, $\mathcal{I}$ is updated as
\begin{equation} \label{eqn: IS to MIS}
\mathcal{I} \leftarrow \mathcal{I} \cup \texttt{dNN}(G[V\setminus (\mathcal{I} \cup N(\mathcal{I}))],\alpha)\:.
\end{equation}
\begin{algorithm}[t]
\caption{Handling Forbidden edges.}
\textbf{Input}: Graph $G=(V,E)$, $B$, $F$ \\
\textbf{Output}: IS $\mathcal{I}$ on $G$ \\
\vspace{1mm}
\small{ 1:} \textbf{initialize} $\mathcal{I}=B$\\
\vspace{1mm}
\small{ 2:} \textbf{while} $F \neq \emptyset$\\
\vspace{1mm}
\small{ 3:} \hspace{2mm} \textbf{select} a pair $(u,v)\in F$, \textbf{initialize} $\text{ReplacementFlag} = 0$ \\ 
\vspace{1mm}
\small{ 4:} \hspace{2mm} \textbf{for all} $q\in \{u,v\}$ \\
\vspace{1mm}
\small{ 5:} \hspace{4mm} \textbf{if} $\exists w\in N(q) \textbf{ s.t. } |\mathcal{I}\cap N(w)| = 1$ \\
\vspace{1mm}
\small{ 6:} \hspace{6mm} \textbf{replace} $q$ by $w$, that is $\mathcal{I}\leftarrow \mathcal{I}\setminus \{q\}$,  $\mathcal{I}\leftarrow \mathcal{I}\cup \{w\}$\\
\vspace{1mm}
\small{ 7:} \hspace{6mm} \textbf{update} $F$, $\text{ReplacementFlag} = 1$  \\
\vspace{1mm}
\small{ 8:} \hspace{6mm} \textbf{break for}  \\
\vspace{1mm}
\small{ 9:} \hspace{2mm} \textbf{if} $\text{ReplacementFlag} = 0$ (no replacement is found) \\
\vspace{1mm}
\small{10:} \hspace{4mm} \textbf{remove}  either $u$ or $v$ depending on their repetitions in $F$\\
\vspace{1mm}
\small{11:} \hspace{4mm} \textbf{update} $\mathcal{I}$ and $F$ \\
\vspace{-3.5mm}
\label{alg: Handlling}

\end{algorithm}
%



\textcolor{black}{\subsection{Solution Improvement by dNNs}}


After applying the community detection algorithm and using Algorithm~\ref{alg: MIS OSS alg} for every resulting subgraph, Algorithm~\ref{alg: Handlling} along with \eqref{eqn: IS to MIS} are used to obtain MIS $\mathcal{I}$. Since high-degree nodes are less likely to be in a large IS, given graph $G$ and solution $\mathcal{I}$, we propose a solution improvement procedure that removes a set of low-degree nodes $\mathcal{U}\subset \mathcal{I}$, such that $|\mathcal{U}| = \lambda $, along with their neighbours $N(\mathcal{U})$ from the graph. We then apply our dNN on the reduced graph $G[V\setminus (\mathcal{U}\cup N(\mathcal{U}))]$ with different initial seed for $s$ as a form of simulated annealing \cite{van1987simulated}. This procedure is iteratively applied where we increase $\lambda$ at every iteration. The best solution is maintained until some stopping criteria is met. The procedure is given in Algorithm~\ref{alg: sol impr}.\par


While a similar criteria is used to select $\mathcal{U}$ in \cite{lamm2016finding}, their method recursively tries all reduction techniques on the reduced graph where $\lambda$ is fixed. 

\begin{algorithm}[t]
\small
\caption{Solution improvement by dNNs}
\textbf{Input}: Graph $G=(V,E)$, Solution $\mathcal{I}$, $\lambda$, $\textrm{IncreaseStep}$ \\
\textbf{Output}: $\mathcal{I}^*$ \\
\vspace{1mm}
\small{ 1:}  \textbf{initialize} $\mathcal{I}^* = \mathcal{I}$ \\
\vspace{1mm}
\small{ 2:}  \textbf{while} stopping criteria is not satisfied \\
\vspace{1mm}
\small{ 3:} \hspace{2mm} \textbf{obtain}  $\mathcal{U}\subset \mathcal{I} : |\mathcal{U}| = \lambda$,$\forall u\in \mathcal{U}, v\in \mathcal{I}\setminus \mathcal{U}, \textrm{d}(u)\leq \textrm{d}(v)$ \\
\vspace{1mm}
\small{ 4:} \hspace{2mm} \textbf{obtain}  $\mathcal{I} \leftarrow \mathcal{I} \cup \texttt{dNN}(G[V\setminus (\mathcal{U} \cup N(\mathcal{U}))],\alpha)$\\
\vspace{1mm}
\small{ 5:} \hspace{2mm} \textbf{if} $|\mathcal{I}| > |\mathcal{I}^*| $ (update the optimal if $\mathcal{I}$ is of higher cardinality)\\
\vspace{1mm}
\small{ 6:}  \hspace{4mm} \textbf{update} $\mathcal{I}^* = \mathcal{I}$  \\
\vspace{1mm}
\small{ 7:} \hspace{2mm} \textbf{if} $|\mathcal{I}| \leq |\mathcal{I}^*| $ (restart from the current optimal) \\
\vspace{1mm}
\small{ 8:}  \hspace{4mm} \textbf{update} $\mathcal{I} = \mathcal{I}^*$ \\
\vspace{1mm}
\small{ 9:} \hspace{2mm} \textbf{update} $\lambda \leftarrow \lambda + \textrm{IncreaseStep}$ \\
\vspace{1mm}
\vspace{-3.5mm}
\label{alg: sol impr}
\end{algorithm}

\section{Experimental Evaluation}




In this section, we evaluate the performance of our proposed method and present comparisons to state-of-the-art methods using synthetic graphs and real-world benchmarks.

\subsection{Setup, Benchmarks, and Baselines}

We process graphs using the NetworkX library \cite{SciPyProceedings_11} and use Tensorflow \cite{abadi2016tensorflow} to construct the dNN $h$. The initial learning rate for the ADAM optimizer is set to $0.1$. We set the probability threshold $\alpha=0.5$ and use degree-based initialization. Experiments justifying our choice are presented in the appendix. 
For community detection, we use the Louvain algorithm \cite{blondel2008fast} with a resolution factor of 1.3 for large-scale low-density graphs and 0.8 for high-density instances. For Algorithm~\ref{alg: sol impr}, we choose $\lambda=5$ and increase it by 1 in every iteration. The algorithm stops when the number of nodes in the reduced graph is below 20. The experiments are run using Python 3 and Intel(R) Core(TM) i9-9940 CPU @ 3.30GHz machine. 

For low-density graphs, we incorporate the inexpensive and non-recursive LP graph reduction presented in \cite{nemhauser1975vertex} prior to performing community detection and constructing the enhanced dNN $h$. A half-integral solution (using values $0$, $1/2$, and $1$), $x^* = \argmax \{ \sum_{v\in V} x_v \text{ s.t. } x_v \geq 0, \forall v\in V, x_v+x_u \leq 1, \forall (u,v)\in E \}$ is obtained using bipartite matching. The vertices that are members of set $T =: \{v\in V \mid x^*_v = 1\}$ must be in the MIS and can thus be removed from $G$ together with their neighbors in $N(T)$. The solution obtained from training $h$ on $G[V\setminus (T\cup N(T))]$ is joined with nodes in $T$ to obtain the MIS for $G$. Furthermore, we implement the 2-improvement basic local search algorithm \cite{andrade2012fast}. The foregoing techniques are also used in most of the state-of-the-art solvers presented in \cite{lamm2016finding, li2018combinatorial, ahn2020learning}.


As a benchmark, we use the social network graphs from the Stanford Large Network Dataset Collection given in SNAP \cite{snapnets}. In these graphs, the vertices represent people and the edges reflect their interactions. We also use the citation network graphs \cite{sen2008collective} for data collected from academic search engines. In these graphs, nodes represent documents and edges reflect their citations. Using the aforementioned benchmarks, we compare the performance of our proposed framework to multiple MIS solvers, including the GCN method \cite{li2018combinatorial}, which is an ML-based approach, 
and the RL-based method S2V-DQN \cite{dai2017learning}. 
We also report results from the state-of-the-art MIS solver ReduMIS \cite{lamm2016finding}. We use the size of the identified independent set to measure the quality of the solution for every baseline considered. Furthermore, results from solving the MIS Integer Linear Program (ILP) in \eqref{eqn: MIS ILP} using CPLEX are also reported.
\begin{equation}
\begin{gathered} \label{eqn: MIS ILP}
\max_{x} \sum_{v \in V(G)} x_v \text{  subject to} \\
x_v \in \{0,1\} \:, \forall v \in V,\:
x_v + x_u \leq 1\:, \forall (v,u) \in E \:.
\end{gathered} 
\end{equation}

The aforementioned benchmarks are considered sparse graphs. Therefore, we also test our proposed method on higher-density graphs randomly generated from the Erdos-Renyi (ER) \cite{erdos1960evolution}, Barbosi-Albert (BA) \cite{albert2002statistical}, Holme and Kim (HK) \cite{holme2002growing}, and the Stochastic Block (SBM) \cite{holland1983stochastic} models. We note that learning-based methods, such as \cite{li2018combinatorial}, are known to only be applicable to sparse graphs. Thus, we compare our performance to ReduMIS and CPLEX in these settings.

\begin{table*}[t]
\caption{{Comparison to state-of-the-art baselines using real-world benchmarks in terms of the size of the identified MIS.}}
\label{tbl: comparison}\centering
 \scalebox{0.85}{\begin{tabular}{||c c c c c c c c||} 
 \hline
 \textbf{Dataset} &  $|V|$ & $|E|$ & GCN  & ReduMIS & S2V-DQN & CPLEX & dNNs \\ [0.5ex] 
 \hline\hline

 \hline
 bitcoin-alpha &  3783 & 14124 &  \textbf{2718}  & \textbf{2718} & 2705 & \textbf{2718} & \textbf{2718} \\ [0.5ex]
 \hline
 bitcoin-otc &  5881 & 21492 &  4346  & 4346 & 4334 & 4346 & \textbf{4347}  \\ [0.5ex]
 \hline
 wiki-Vote &  7115 & 100762 &  \textbf{4866} & \textbf{4866} & 4779 & \textbf{4866} & \textbf{4866} \\ [0.5ex]
 \hline
   soc-slashdot0811 &  73399 & 497274  & \textbf{53314} & \textbf{53314} & 52719 & \textbf{53314} & \textbf{53314} \\ [0.5ex]
 \hline
   soc-slashdot0922 &  82168 & 582533 &  \textbf{56398} & \textbf{56398} & 55506 & \textbf{56398} & 56395  \\ [0.5ex]
 \hline
    soc-Epinions1 &  75579 & 405740  &  \textbf{53599} & \textbf{53599} & 53089 & \textbf{53599} & 53598  \\ [0.5ex]
 \hline\hline
     Citeseer &  3327 & 4536 &  \textbf{1867} & \textbf{1867} & 1705 & 1808 & 1866   \\ [0.5ex]
 \hline
     Cora  &  2708 & 5429 &  \textbf{1451} & \textbf{1451} & 1381 & \textbf{1451} & \textbf{1451}  \\ [0.5ex]
 \hline
      PubMed  &  19717 & 44327 &  \textbf{15912} & \textbf{15912} & 15709 & \textbf{15912} & \textbf{15912}  \\ [0.5ex]
 \hline
  \end{tabular}}
\vspace{ -0.25cm}
\end{table*}



\begin{table}[t]
\caption{{Comparison to state-of-the-art baselines using synthetic graphs in terms of the average size of the MIS.}}
\label{tbl: synthetic graphs comparison}\centering
 \scalebox{0.85}{\begin{tabular}{||c c c c c c ||} 
 \hline
 \textbf{Graph Type} &  $|V|$ & $\mathbb{E}(|E|) $ & ReduMIS & CPLEX & dNNs \\ [0.5ex] 
 \hline\hline
 ER &  100 ($p=0.1$) & 496 &   \textbf{30.5}  & \textbf{30.5} & \textbf{30.5}  \\ [0.5ex]
 \hline
 ER &  100 ($p=0.2$) & 975 &   \textbf{20}  & \textbf{20} & \textbf{20}  \\ [0.5ex]
 \hline
 ER &  200 ($p=0.1$) & 1991.5 &   \textbf{41}  & \textbf{41} & \textbf{41}  \\ [0.5ex]
 \hline
   ER &  200 ($p=0.2$) & 3983.5 &   \textbf{25.5}  & \textbf{25.5} & \textbf{25.5}  \\ [0.5ex]
     \hline\hline
 SBM &  250 ($p=0.1$) & 1857.5  &  57.5 & 60.5 & \textbf{61} \\ [0.5ex]
 \hline
     SBM &  250 ($p=0.2$) & 2431  &  47 & \textbf{51} & \textbf{51} \\ [0.5ex]
 \hline
  SBM &  350 ($p=0.1$) & 3614.5  &  64 & 66.5 & \textbf{68} \\ [0.5ex]
 \hline
     SBM &  350 ($p=0.2$) & 4826  & 52 & 53.5 & \textbf{55.5} \\ [0.5ex]
 
 \hline\hline
    BA &  100 & 2450  &  \textbf{45} & \textbf{45} & \textbf{45} \\ [0.5ex]
 \hline
     BA &  200 & 9950  &  \textbf{95} & \textbf{95} & \textbf{95} \\ [0.5ex]
 \hline\hline
 HK &  100 & 2500  &  \textbf{30} & \textbf{30} & \textbf{30} \\ [0.5ex]
 \hline
     HK &  200 & 9900  &  \textbf{60} & \textbf{60} & \textbf{60} \\ [0.5ex]
 \hline
 
  \end{tabular}}
\vspace{ -0.25cm}
\end{table}

\subsection{Results on SNAP and Citation Network Benchmarks}
In this subsection, we present the overall comparison results with the GCN, ReduMIS, and S2V-DQN methods along with the results from solving the MIS ILP in \eqref{eqn: MIS ILP} using CPLEX. Columns 4 to 8 of Table~\ref{tbl: comparison} present the size of the found MIS. The results, other than the CPLEX ILP, reported for the baselines are obtained from Table 5 of \cite{li2018combinatorial}.

We briefly describe the reduction techniques utilized by these baselines as they contribute significantly to their final results. All three methods remove pendent, unconfined, and twin vertices and utilize vertex folding for degree-2 nodes. Additional reductions are also considered in ReduMIS, including finding node alternatives, using packing constraints, and adopting the same LP MIS relaxation reduction we adopt in this paper. We refer the interested reader to \cite{akiba2016branch} for a thorough discussion of these reductions. 

In all datasets considered, our method outperforms S2V-DQN. 
Our method performs mostly on par with ReduMIS and GCN, both of which yield identical results. This is observed as exact MIS sizes are obtained for bitcoin-alpha, Wiki-Vote, soc-slashdot0811, Cora, and PubMed. While scoring higher for bitcoin-otc, our method scores lower for soc-slashdot0922, soc-Epinions1, and Citeseer. When compared to the ILP solver, our method outperforms CPLEX on bitcoin-otc and Citeseer while performing on par for all the other instances other than soc-slashdot0922 and soc-Epinions.


\subsection{Results on Synthetic Graphs}
In this section, we compare our proposed method to ReduMIS and CPLEX using random graphs generated from the ER, BA, and HK models. Every size reported in Table~\ref{tbl: synthetic graphs comparison} represents the average of two random graphs from the model given in the first column. For ER, $p$ represents the probability of an edge being present. For BA, we use $\lfloor 0.5n \rfloor$ and $\lfloor 0.45n \rfloor$ edges to attach a new node to existing nodes. For HK, we add $\lfloor 0.3n \rfloor$ 
random edges to each new node 
and set the probability of adding a triangle after adding a random edge to $0.5$. For SBM, we generate graphs with 5 clusters, where two nodes from the same cluster share an edge with probability $p$ (intra-cluster density) and two nodes from different clusters share an edge with probability $q= 0.05$ (inter-cluster density). 
The random graph parameters are selected to yield high-density graphs relative to the earlier benchmarks. As shown, in all the considered random high-density graphs, we perform on par with both CPLEX and ReduMIS on average. For clustered graphs from the SBM, we outperform both CPLEX and ReduMIS.

\section{Conclusion}
We presented a dataless differentiable methodology for solving \textbf{NP-hard} problems that is radically different from existing techniques. The underpinning of our approach is a reduction from the Maximum Independent Set (MIS) problem to an equivalent dataless Neural Network (dNN) constructed from the given graph. The parameters of this dNN are trained without requiring data, thereby setting our approach apart from learning-based methods like supervised, unsupervised, and reinforcement learning. In particular, training is conducted by applying backpropagation to a loss function defined entirely based on the structure of the given graph. We also presented an enhanced version of the dNN by incorporating the edges from the complement graph and exploiting the duality of the Maximum Clique (MC) and MIS problems. Additionally, 
we developed a reduction procedure that leverages a community detection algorithm to scale our approach to larger and higher-density graphs. Unlike previous reductions, the procedure is independent of the type of the graph and its density. 
Experimental results on real-world benchmarks demonstrate that our proposed method performs on par with state-of-the-art learning-based methods without requiring any training data. Furthermore, for higher-density graphs, where learning-based methods are not applicable, we have shown that our method performs on par with, or outperforms, the state-of-the-art methods.




\bibliographystyle{IEEEbib}
\bibliography{ref}

\begin{thebibliography}{10}

\bibitem{karp1972reducibility}
Richard~M Karp,
\newblock ``Reducibility among combinatorial problems,''
\newblock in {\em Complexity of computer computations}, pp. 85--103. Springer,
  1972.

\bibitem{bengio2021machine}
Yoshua Bengio, Andrea Lodi, and Antoine Prouvost,
\newblock ``Machine learning for combinatorial optimization: a methodological
  tour d’horizon,''
\newblock {\em European Journal of Operational Research}, vol. 290, no. 2, pp.
  405--421, 2021.

\bibitem{li2018combinatorial}
Zhuwen Li, Qifeng Chen, and Vladlen Koltun,
\newblock ``Combinatorial optimization with graph convolutional networks and
  guided tree search,''
\newblock {\em Advances in Neural Information Processing Systems}, p. 539,
  2018.

\bibitem{lamm2016finding}
Sebastian Lamm, Peter Sanders, Christian Schulz, Darren Strash, and Renato~F
  Werneck,
\newblock ``Finding near-optimal independent sets at scale,''
\newblock in {\em 2016 Proceedings of the Eighteenth Workshop on Algorithm
  Engineering and Experiments (ALENEX)}. SIAM, 2016, pp. 138--150.

\bibitem{akiba2016branch}
Takuya Akiba and Yoichi Iwata,
\newblock ``Branch-and-reduce exponential/fpt algorithms in practice: A case
  study of vertex cover,''
\newblock {\em Theoretical Computer Science}, vol. 609, pp. 211--225, 2016.

\bibitem{san2011exact}
Pablo San~Segundo, Diego Rodr{\'\i}guez-Losada, and Agust{\'\i}n Jim{\'e}nez,
\newblock ``An exact bit-parallel algorithm for the maximum clique problem,''
\newblock {\em Computers \& Operations Research}, vol. 38, no. 2, pp. 571--581,
  2011.

\bibitem{hochba1997approximation}
Dorit~S Hochba,
\newblock ``Approximation algorithms for np-hard problems,''
\newblock {\em ACM Sigact News}, vol. 28, no. 2, pp. 40--52, 1997.

\bibitem{he2014learning}
He~He, Hal Daume~III, and Jason~M Eisner,
\newblock ``Learning to search in branch and bound algorithms,''
\newblock {\em Advances in neural information processing systems}, vol. 27, pp.
  3293--3301, 2014.

\bibitem{dai2016discriminative}
Hanjun Dai, Bo~Dai, and Le~Song,
\newblock ``Discriminative embeddings of latent variable models for structured
  data,''
\newblock in {\em International conference on machine learning}. PMLR, 2016,
  pp. 2702--2711.

\bibitem{boppana1992approximating}
Ravi Boppana and Magn{\'u}s~M Halld{\'o}rsson,
\newblock ``Approximating maximum independent sets by excluding subgraphs,''
\newblock {\em BIT Numerical Mathematics}, vol. 32, no. 2, pp. 180--196, 1992.

\bibitem{williamson2011design}
David~P Williamson and David~B Shmoys,
\newblock {\em The design of approximation algorithms},
\newblock Cambridge university press, 2011.

\bibitem{hastad1996clique}
Johan Hastad,
\newblock ``Clique is hard to approximate within n/sup 1-/spl epsiv,''
\newblock in {\em Proceedings of 37th Conference on Foundations of Computer
  Science}. IEEE, 1996, pp. 627--636.

\bibitem{berman1992complexity}
Piotr Berman and Georg Schnitger,
\newblock ``On the complexity of approximating the independent set problem,''
\newblock {\em Information and Computation}, vol. 96, no. 1, pp. 77--94, 1992.

\bibitem{dai2017learning}
Hanjun Dai, Elias~B Khalil, Yuyu Zhang, Bistra Dilkina, and Le~Song,
\newblock ``Learning combinatorial optimization algorithms over graphs,''
\newblock in {\em Proceedings of the 31st International Conference on Neural
  Information Processing Systems}, 2017, pp. 6351--6361.

\bibitem{defferrard2016convolutional}
Micha{\"e}l Defferrard, Xavier Bresson, and Pierre Vandergheynst,
\newblock ``Convolutional neural networks on graphs with fast localized
  spectral filtering,''
\newblock {\em Advances in neural information processing systems}, vol. 29, pp.
  3844--3852, 2016.

\bibitem{andrade2012fast}
Diogo~V Andrade, Mauricio~GC Resende, and Renato~F Werneck,
\newblock ``Fast local search for the maximum independent set problem,''
\newblock {\em Journal of Heuristics}, vol. 18, no. 4, pp. 525--547, 2012.

\bibitem{riedmiller1993direct}
Martin Riedmiller and Heinrich Braun,
\newblock ``A direct adaptive method for faster backpropagation learning: The
  rprop algorithm,''
\newblock in {\em IEEE International Conference on Neural Networks}, 1993, pp.
  586--591.

\bibitem{cook1995combinatorial}
William Cook, L{\'a}szl{\'o} Lov{\'a}sz, Paul~D Seymour, et~al.,
\newblock {\em Combinatorial optimization: papers from the DIMACS Special
  Year}, vol.~20,
\newblock American Mathematical Soc., 1995.

\bibitem{kingma2014adam}
Diederik~P Kingma and Jimmy Ba,
\newblock ``Adam: A method for stochastic optimization,''
\newblock {\em arXiv preprint arXiv:1412.6980}, 2014.

\bibitem{goodfellow2017deep}
Ian Goodfellow, Yoshua Bengio, and Aaron Courville,
\newblock ``Deep learning (adaptive computation and machine learning series),''
\newblock {\em Cambridge Massachusetts}, pp. 321--359, 2017.

\bibitem{yang2016comparative}
Zhao Yang, Ren{\'e} Algesheimer, and Claudio~J Tessone,
\newblock ``A comparative analysis of community detection algorithms on
  artificial networks,''
\newblock {\em Scientific reports}, vol. 6, no. 1, pp. 1--18, 2016.

\bibitem{van1987simulated}
Peter~JM Van~Laarhoven and Emile~HL Aarts,
\newblock ``Simulated annealing,''
\newblock in {\em Simulated annealing: Theory and applications}, pp. 7--15.
  Springer, 1987.

\bibitem{SciPyProceedings_11}
Aric~A. Hagberg, Daniel~A. Schult, and Pieter~J. Swart,
\newblock ``Exploring network structure, dynamics, and function using
  networkx,''
\newblock in {\em Proceedings of the 7th Python in Science Conference}, Ga\"el
  Varoquaux, Travis Vaught, and Jarrod Millman, Eds., Pasadena, CA USA, 2008,
  pp. 11 -- 15.

\bibitem{abadi2016tensorflow}
Mart{\'\i}n Abadi, Paul Barham, Jianmin Chen, Zhifeng Chen, Andy Davis, Jeffrey
  Dean, Matthieu Devin, Sanjay Ghemawat, Geoffrey Irving, Michael Isard,
  et~al.,
\newblock ``Tensorflow: A system for large-scale machine learning,''
\newblock in {\em 12th $\{$USENIX$\}$ Symposium on Operating Systems Design and
  Implementation ($\{$OSDI$\}$ 16)}, 2016, pp. 265--283.

\bibitem{blondel2008fast}
Vincent~D Blondel, Jean-Loup Guillaume, Renaud Lambiotte, and Etienne Lefebvre,
\newblock ``Fast unfolding of communities in large networks,''
\newblock {\em Journal of statistical mechanics: theory and experiment}, vol.
  2008, no. 10, pp. P10008, 2008.

\bibitem{nemhauser1975vertex}
George~L Nemhauser and Leslie~Earl Trotter,
\newblock ``Vertex packings: Structural properties and algorithms,''
\newblock {\em Mathematical Programming}, vol. 8, no. 1, pp. 232--248, 1975.

\bibitem{ahn2020learning}
Sungsoo Ahn, Younggyo Seo, and Jinwoo Shin,
\newblock ``Learning what to defer for maximum independent sets,''
\newblock in {\em International Conference on Machine Learning}. PMLR, 2020,
  pp. 134--144.

\bibitem{snapnets}
Jure Leskovec and Andrej Krevl,
\newblock ``{SNAP Datasets}: {Stanford} large network dataset collection,''
  \hbox{http://snap.stanford.edu/data}, June 2014.

\bibitem{sen2008collective}
Prithviraj Sen, Galileo Namata, Mustafa Bilgic, Lise Getoor, Brian Galligher,
  and Tina Eliassi-Rad,
\newblock ``Collective classification in network data,''
\newblock {\em AI magazine}, vol. 29, no. 3, pp. 93--93, 2008.

\bibitem{erdos1960evolution}
Paul Erdos, Alfr{\'e}d R{\'e}nyi, et~al.,
\newblock ``On the evolution of random graphs,''
\newblock {\em Publ. Math. Inst. Hung. Acad. Sci}, vol. 5, no. 1, pp. 17--60,
  1960.

\bibitem{albert2002statistical}
R{\'e}ka Albert and Albert-L{\'a}szl{\'o} Barab{\'a}si,
\newblock ``Statistical mechanics of complex networks,''
\newblock {\em Reviews of modern physics}, vol. 74, no. 1, pp. 47, 2002.

\bibitem{holme2002growing}
Petter Holme and Beom~Jun Kim,
\newblock ``Growing scale-free networks with tunable clustering,''
\newblock {\em Physical review E}, vol. 65, no. 2, pp. 026107, 2002.

\bibitem{holland1983stochastic}
Paul~W Holland, Kathryn~Blackmond Laskey, and Samuel Leinhardt,
\newblock ``Stochastic blockmodels: First steps,''
\newblock {\em Social networks}, vol. 5, no. 2, pp. 109--137, 1983.

\end{thebibliography}

\end{document}